\documentclass{article}

\usepackage{microtype}
\usepackage{graphicx}
\usepackage{subcaption}
\usepackage{booktabs} %

\usepackage{hyperref}

\usepackage[preprint]{icml2026}

\usepackage{amsmath}
\usepackage{amssymb}
\usepackage{mathtools}
\usepackage{amsthm}

\usepackage[capitalize,noabbrev]{cleveref}

\theoremstyle{plain}
\newtheorem{theorem}{Theorem}[section]

\newtheorem{lemma}[theorem]{Lemma}

\theoremstyle{definition}

\theoremstyle{remark}

\newtheorem{example}[theorem]{Example}

\usepackage[textsize=tiny]{todonotes}

\icmltitlerunning{Post-Norm can Resharpen Attention}

\begin{document}

\twocolumn[
  \icmltitle{Post-Norm can Resharpen Attention}

  \icmlsetsymbol{equal}{*}

\begin{icmlauthorlist}
  \icmlauthor{Pál Zsámboki}{renyi}
  \icmlauthor{Benjamin Levi}{rochester}
  \icmlauthor{David Ansel Josef Smith}{alabama}
  \icmlauthor{Mitansh Kagalwala}{virginia}
  \icmlauthor{Arlington Kell}{gatech,equal}
  \icmlauthor{Samuel Liechty}{byu,equal}
  \icmlauthor{Cong Wang}{carleton}
\end{icmlauthorlist}

\icmlaffiliation{renyi}{HUN-REN Alfréd Rényi Institute of Mathematics, Budapest, Hungary}
\icmlaffiliation{rochester}{University of Rochester, Rochester, NY, USA}
\icmlaffiliation{alabama}{The University of Alabama, Tuscaloosa, AL, USA}
\icmlaffiliation{virginia}{University of Virginia, Charlottesville, VA, USA}
\icmlaffiliation{gatech}{Georgia Institute of Technology, Atlanta, GA, USA}
\icmlaffiliation{byu}{Brigham Young University, Provo, UT, USA}
\icmlaffiliation{carleton}{Carleton College, Northfield, MN, USA}

\icmlcorrespondingauthor{Pál Zsámboki}{zsamboki@renyi.hu}

  \icmlkeywords{attention dispersion, length generalization, post-norm, transformers}

  \vskip 0.3in
]

\printAffiliationsAndNotice{\icmlEqualContribution}

\begin{abstract}
Length Generalization is the essential capacity of autonomous agents to perform tasks in longer contexts than those encountered during training. To systematically study this feat, we test how well models can approximate the next token distributions in algorithmic tasks. This is to take into account the realistic possibility of multiple next tokens being legal. We present a prototypical benchmark for this line of study: in the Set Complement Task, the model needs to output a uniform distribution over tokens not in the input. We prove a theorem that states simple transformers can length generalize on this task, however, with performance degradation due to attention dispersion. A mechanistic reading of how dispersion takes effect lets us discover a remedy: Post-Norm can Resharpen Attention. We present experimental evidence to support this idea. We also show that Exponential Moving Averages can help the issue of noisy gradients that arises when many next tokens are legal. We validate the general app\-licability of our proposed methods on a suite of formal language experiments. Our source code will be available upon publication.
\end{abstract}

\section{Introduction}

Since it was discovered that transformer \cite{vaswani2017attention}-based large language models (LLMs) can be aligned with human interests \cite{ouyang2022training}, LLM agents are being deployed in roles as diverse as application developers, counsellors, job interviewers, or research assistants. For both safety and efficiency it is paramount that we understand how these agents make their decisions.

Prior work revealed that to improve the reasoning capabilities of transformers, one can prompt \cite{wei2022chain} or train them to think in small steps. This makes reasoning transformers produce their answers akin to game-playing agents that map out various potential trajectories before making a move.

Therefore, we can gain insights on the reasoning processes of LLM agents by studying how models with similar architecture learn to play games. In the present work, we focus on the most fundamental skill an agent playing a game as simple as tic-tac-toe or as complex as go has to acquire: tell which board positions are not yet taken.

We abstract this task as the Set Complement Task, introduced in Subsection \ref{subsection:set complement task}: given an input sequence of tokens without repetition, the model has to output a uniform distribution on the tokens absent from the input sequence. Note that we aim for more that top-1 accuracy, that is for the model to predict as most probable next token one that is missing from the input: such a basic component has to be learnt free from bias, that is, without predicting one valid token more often than the other.

Our theoretical contribution is Theorem \ref{theorem:bounds and length generalization}: a characterization of single-layer, attention-only, uniform attention models that can learn the task. First of all, we give tight bounds for the embedding and value dimensions required of the model. Second, we show that if the model can solve the task on input sequences of length 1 and 2, then it can solve the task for input sequences of any lengths, albeit at reduced precision.

This connects our work to the topic of length generalization: if a model robustly learned to perform an algorithmic task, then it should be able to produce a correct output on input sequences longer than those in its training set. It is an active field of study which tasks transformers can length generalize on: we discuss this in detail in Section \ref{section:related work}.

Overcoming obstacles to length generation is of particular interest for this line of study. In Subsection \ref{subsection:post-norm can resharpen attention}, a mechanistic reading of the way our models make their inferences identifies the reduction in precision as a particular case of attention dispersion \cite{hahn2020theoretical,velickovic2025softmax}: by taking a mean along the sequential dimention, softmax attention reduces the relative differences between attention weights as sequence length increases. We hypothesize that normalization of the output of the attention blocks may be able to undo the reduction in the amplitude of value vectors and thus mitigate this effect.

We turn to the study of training dynamics in Subsection \ref{subsection:bema}. We use conventional next token logit prediction training via negative log likelihood of one-hot sampled target distributions. Mechanistic analysis of training unveils a further obstacle for length generalization: in our task, the tokens that follow short sequences are sampled from many possibilities, which makes gradients noisy. Therefore, in Subsection \ref{subsection:bema}, we make our second hypothesis: using the stabilizer Bias-corrected Exponential Moving Average (BEMA) may attenuate this effect.

We investigate the effect of our proposed strategies in random hyperparameter search experiments. We describe our experimental protocol in Section \ref{section:experimental setup}. We report our experiment results in Section \ref{section:experiment results}: in Subsection \ref{subsection:sct results}, we provide experimental evidence that our proposed methods indeed foster length generalization in the Set Complement Task. We show that our methods have an effect in further length generalization tasks: Maximum Retrieval in Subsection \ref{subsection:max retrieval results}, a bounded Dyck language in Subsection \ref{subsection:dyck results}, and Tomita languages 3 and 7 in \ref{subsection:tomita results}.

\section{Related Work}\label{section:related work}

\emph{Length Generalization} studies the conditions under which sequen\-ce-to-sequence models retain their performance on inputs longer than those seen during training. One train of results seeks to find criteria for algorithms transformers can length generalize on \cite{bhattamishra2020ability}. An important theoretical tool in this direction is the Restricted Access Sequence Processing Language (RASP) \cite{rush2023thinking}, a programming language that a transformer can implement. It was conjectured that length generalization is possible if there is a simple implementation in RASP-L \cite{zhou2024what}. Afterwards, a version of this conjecture was proven \cite{huang2025a} using limit transformers, and a version of the C-RASP language \cite{yang2024counting}. In the usual algorithmic approach to the study of length generalization, if multiple solutions are possible, then the model has to predict the set of valid solutions as a singleton. We aim to bring in an alternative point of view closer to the spirit of language modeling: if multiple solutions are possible, the model should learn to output the correct distribution between them.

\emph{Attention Dispersion} \cite{hahn2020theoretical,velickovic2025softmax} is a drawback of softmax attention that forbids generalization to arbitrary sequence lengths both in toy and language models. Multiple previous works considered fixing this dilution by an explicit, length-dependent transformation of the attention formula, either a $\log(s)$ multiplier on the attention logits \cite{chiang2022overcoming}, or a fixed temperature function found by polynomial interpolation \cite{velickovic2025softmax,peng2024yarn}. However, to our knowledge, in the context of length generalization, post-normalization was only considered theoretically \cite{hahn2020theoretical} or in custom-made constructions \cite{yao2021self}, not as a general helper.

In the original transformer \cite{vaswani2017attention}, \emph{normalization layers} are put on the residual stream after the attention and the feedforward block updates have been applied. This approach is called called \emph{post-normalization}. Later on, it was discovered that this makes training unstable at initialization \cite{xiong2020layer}, so normalizing the inputs to the attention and feedforward blocks was proposed, called \emph{pre-normalization}. However, it was observed that pre-normalization can cause massive activations to appear later in training \cite{sun2024massive}. Further options were considered in specific architectures, for example: normalizing both the input and the output of the residual blocks as \emph{Sandwich-LN} in the text-to-image model Cogview \cite{ding2021cogview}, and normalizing the output of the residual connections in the language models Swin Transformer 2 V2 \cite{liu2022swin}. These ideas were combined to \emph{peri-normalization} \cite{kim2025periln}, where the inputs and outputs of the residual blocks, the input of the unembedding, and optionally the output of the embedding are normalized. In all these works, normalization is studied from the point of view of stable training, not length generalization.

\emph{Mechanistic Interpretability} aims to find minimal subnetworks, so-called \emph{circuits}, of an artificial neural network that satisfy a given task. In the case of transformers, several such circuits have been identified such as induction heads \cite{elhage2021mathematical}, indirect object identification circuits \cite{wang2023interpretability}, greater-than circuits \cite{hanna2023how}, and retrieval heads \cite{wu2025retrieval}. Of particular interest to the present work are the studies on OthelloGPT, which showed that in the residual stream of a GPT-1 style model trained to predict legal moves on random Othello games, via nonlinear \cite{li2023emergent} and linear \cite{nanda2023emergent} probing, one can find representations of board state. We intend to extend the compendium of known circuits by minimal transformers that can solve the Set Completion Task.

The study of the \emph{next token distributions} output by LLMs brings a detailed view on how the models generate theirs answers, and how expressive they can get. An important part of this point of view is how \emph{calibrated} are the models, that is how well do the predicted next token distributions approximate the target next token distribution. It is shown \cite{shlegeris2024language} that pretrained models as small as GPT-Neo-1.3B surpass humans in next token prediction on the OpenWebText dataset \cite{gokaslan2019openweb}, both in top-1 accuracy, and perplexity. However, calibration to the pretraining corpus can be proven to bring in hallucinations \cite{kalai2024calibrated}, at the very least on facts not present in the training dataset, given the assumption that there are exponentially more ways to complete a sentence in an untruthful way. Neither base or aligned models are calibrated in numeric contexts such as generating tokens from a uniform distribution \cite{lovering2025language}, rather they have strong systematic biases such as dependence on token order. Through soft \cite{li2021prefix} and hard \cite{wallace2019universal} prompt tuning experiments, it was discovered \cite{wang2025distribution} that transformers are more capable of outputting distributions of very low or very high entropy, those with outliers, or those that were output by other transformers. The experiments were conducted both on pretrained and randomly initialized models, thus indicating that the limits in expressivity may stem from the transformer architecture, or the softmax output. In our work, we also investigate if the model learns the true distribution among next tokens, thus indicating that predictions are free from bias.

\section{Preliminaries}

In the first, theoretical part of the paper, we will study single layer, attention only transformers trained on the Set Complement Task. In this Section, we shall introduce the requisite notions. In the second half of the paper, we will report on experiments on full transformers trained on more complex tasks.

\subsection{The Set Complement Task} \label{subsection:set complement task}

In what follows, we shall introduce the \emph{set complement task} that the models we interpret are trained on. To put it very succintly, the models are required to output a uniform distribution over tokens absent from an input without repetitions. Let us formalize this setting.

Let $v$ denote the number of distinct tokens. As they are only meant to signify the $v$ distinct elements of a finite set, we will denote tokens by integers. We let the \emph{vocabulary} or \emph{ambient set} be the finite set $\mathbb V=\{1,\dotsc,v\}$ of $v$ of distinct tokens. We call $v$ the \emph{vocabulary size} or \emph{ambient set size}. As it is not our focus here, we will forego using special tokens such as beginning of sequence, end of sequence, or padding. The valid input sequences are sequences $\mathbf t=(t_1,\dotsc,t_s)\in\mathbb V^s$ of length $1\le s<v$ without repetitions: for distinct indices $1\le i\ne j\le s$, we have $t_i\ne t_j$. The \emph{underlying set} of $\mathbf t$ is the set $[\mathbf t]=\{t_1,\dotsc,t_s\}$ of tokens in $\mathbf t$. We let $|\mathbf t|=s$ denote the length of $\mathbf t$.

We represent the set of categorical distributions on $v$ entries as the \emph{probability $(v-1)$-simplex}
$$
\Delta^{v-1}=\left\{\mathbf p\in\mathbb R^v_{\ge0}:\sum_{t=1}^v p_t=1\right\}.
$$
Let $\mathbb X$ denote the set of valid input sequences. Then the perfect solution to the task is the function $p^*:\mathbb X\to\Delta^{v-1}$ such that, for any input sequence $\mathbf t\in\mathbb X$ and token $t\in\mathbb V$, we have
\begin{equation}\label{equation:perfect solution}
p^*(\mathbf t)_t=\begin{cases}
0 & i\in[\mathbf t],\\
\frac{1}{v-s} & i\notin[\mathbf t].
\end{cases}
\end{equation}

\subsection{Minimal Transformers}\label{subsection:minimal transformers}

We will seek to approximately solve the set complement task with parametric models of the form
\begin{equation}\label{equation:softmax model}
\mathbb X\xrightarrow{f_\theta}\mathbb R^v\xrightarrow{\mathrm{softmax}}\Delta^{v-1},
\end{equation}
where $f_\theta$ denotes a single-layer, attention-only, single-head, decoder-only transformer with parameter vector $\theta$. Let $\mathbf t=(t_1,\dots,t_s)\in\mathbb X$ be an input sequence. We call the output $f_\theta(\mathbf t)\in\mathbb R^{v}$ the \emph{next token logit vector}. The next token logit vector is formed as the linear combination
\begin{equation}\label{equation:next token logit vector}
f_\theta(\mathbf t)
=\mathbf B_{t_s,:}+\sum_{i=1}^s\frac{ a_{t_s,t_i}}{\sum_{i'=1}^s a_{t_s,t_{i'}}}\mathbf D_{t_i,:},
\end{equation}
the terms in which are defined as follows:

The \emph{next token logit bias matrix} $\mathbf B=\mathbf E\mathbf U\in\mathbb R^{v\times v}$ is the product of the \emph{token embedding parameter matrix} $\mathbf E\in\mathbb R^{v\times d}$ and the \emph{unembedding parameter matrix} $\mathbf U\in\mathbb R^{d\times v}$. For any $1\le t\le v$, the $t$-th row $\mathbf E_{t,:}\in\mathbb R^d$ is a $d$-dimensional trainable vector the token $t\in\mathbb V$ is mapped to. We call $d$ the \emph{embedding dimension}, the vector space $\mathbb R^d$ the \emph{residual stream}, and the vector space $\mathbb R^v$ the \emph{logit space}. In our minimal transformers, we do not use positional encodings. 

The \emph{unnormalized per-token attention weight matrix} $\mathbf A$ is formed as follows: We get the \emph{query and key per-token matrices} $\mathbf Q=\mathbf E'\mathbf W_Q$, $\mathbf K=\mathbf E'\mathbf W_K$ via the \emph{query and key parameters matrices} $\mathbf W_Q,\mathbf W_K\in\mathbb R^{d\times d_k}$. We call $d_k$ the \emph{key dimension}. In our setup, the key dimension is not necessarily equal to the embedding dimension divided by the number of attention heads. The \emph{per-token attention logit matrix} is the product $\mathbf A'=\mathbf Q\mathbf K^T$. This yields the unnormalized per-token attention weight matrix via the elementwise formula: $ a_{i,j}=\exp( a'_{i,j}/\sqrt{d_k})$. Note that as we define the output in Equation (\ref{equation:next token logit vector}) for one input sequence only, we do not have to be explicit about causal attention.

The \emph{next token logit displacement matrix} $\mathbf D=\mathbf E\mathbf W_V\mathbf W_O\mathbf U$ is formed via the \emph{value and output parameter matrices} $\mathbf W_V\in\mathbb R^{d\times d_v}$, $\mathbf W_O\in\mathbb R^{d_v\times d}$. We call $d_v$ the \emph{value dimension}. In our setup, the value dimension is not necessarily equal to the key dimension, nor is it necessarily equal to the embedding dimension divided by the number of attention heads.

\section{Theoretical Analysis}

We say that the model $f_\theta\colon\mathbb X\to\mathbb R^v$ \emph{has precision $C>0$} if for all input sequences $\mathbf t\in\mathbb X$, and tokens $u\in\mathbb V\setminus [\mathbf t]$, $v\in\mathbb V$, we have
\begin{equation}\label{equation:has precision}
f_\theta(\mathbf t)_u - f_\theta(\mathbf t)_v \begin{cases}
> C & v\in[\mathbf t],\\
= 0 & v\in\mathbb X\setminus [\mathbf t].
\end{cases}
\end{equation}
We say that the model \emph{has precision $C>0$ at (resp.~up to) length $s$}, if the above property (\ref{equation:has precision}) is satisfied for input sequences $\mathbf t\in\mathbb X$ of length $|\mathbf t|= s$ (resp.~$\le s$).

\subsection{A Hardcoded, Minimal Solution}

Let us first provide a hardcoded model that is precise up to level $D$. In Theorem \ref{theorem:bounds and length generalization}, we will show that its embedding and key dimensions $v-1$ are actually the smallest possible dimensions with which it is possible to solve the task.

\begin{example}\label{example:hardcoded minimal example}
For any vocabulary size $v$, we now give a formula for a model that is arbitrarily close to being perfect if we choose $C>0$ large enough. It uses embedding and value dimensions $d=d_v=v-1$, and key dimension $d_k=1$.

We can use the parameter matrices $\mathbf E=(I|-\boldsymbol1)^T$, $\mathbf U=(-I|\boldsymbol0)$, $\mathbf W_Q=\mathbf W_K=\boldsymbol0$, $
\mathbf W_V=vC\mathbf I$, and $\mathbf W_O=I$.

Note that with these parameters, we get
\begin{gather*}
\mathbf B=\begin{pmatrix}
-1 & 0 & \dotsb & 0 & 0 \\
0 & -1 & \dotsb & 0 & 0 \\
\vdots & \vdots & \ddots & \vdots & \vdots \\
0 & 0 & \dotsb & -1 & 0 \\
1 & 1 & \dotsb & 1 & 0
\end{pmatrix},\,
\mathbf A=\boldsymbol1,\,\\
\mathbf D=\begin{pmatrix}
-vC & 0 & \dotsb & 0 & 0 \\
0 & -vC & \dotsb & 0 & 0 \\
\vdots & \vdots & \ddots & \vdots & \vdots \\
0 & 0 & \dotsb & -vC & 0 \\
vC & vC & \dotsb & vC & 0
\end{pmatrix}.
\end{gather*}
Thus, one can check by hand that $f_\theta$ has precision $C$.
\end{example}

Note that the hardcoded model $f_\theta$ has constant attention. Since when training transformers with weight decay induces low-rank attention logit matrices \cite{kobayashi2024weight}, we will continue our theoretical investigation with the assumption of constant attention. Therefore, the formula (\ref{equation:next token logit vector}) for the next token logit vector $f_\theta(\mathbf t)$ at the input sequence $\mathbf t=(t_1,\dotsc,t_s)\in\mathbb X$ simplifies to the following:
\begin{equation}\label{equation:next token logit vector with constant attention}
f_\theta(\mathbf t)=\mathbf B_{t_s,:}+\frac{1}{s}\sum_{i=1}^s\mathbf D_{t_i,:}.
\end{equation}

\subsection{Length Generalization at the Price of Less Precision}

In this Subsection, we prove tight bounds on the embedding and value dimensions of a constant attention model $f_\theta$. Moreover, we show that if $f_\theta$ approximates the ideal solution on lengths 1 and 2, and moreover it satisfies a balance criterion on token displacements, then it length generalizates, albeit with decreasing precision as length increases.

\begin{theorem}\label{theorem:bounds and length generalization} Assume that the model $f_\theta$ has constant attention. Then the following statements hold:
  
(a) Suppose that the model $f_\theta$ has precision $C>0$ at length 1. Then the matrix $\mathbf B + \mathbf D$ has rank at least $v-1$. In particular, we have $d\ge v-1$.

(b) Suppose moreover that the model $f_\theta$ also has precision $C>0$ at length 2. Then the matrix $\mathbf D$ also has rank at least $v-1$. In particular, we have $d_v\ge v-1$.

(c) Suppose moreover that the following condition is satisfied: for all pairs of distinct tokens $t,u\in\mathbb V$, we have
  \begin{equation}\label{equation:precision upper bound}
    f_\theta((t))_u-f_\theta((t))_t < 2C.
  \end{equation}
  Then for each $3\le s<v$, the model $f_\theta$ has precision $\frac{2}{ s}C$ at length $ s$.

\end{theorem}

See Appendix \ref{proof of Theorem} for the Proof.

\subsection{Post-Norm can Resharpen Attention}\label{subsection:post-norm can resharpen attention}

Inspection of formula (\ref{equation:next token logit vector with constant attention}) shows how precision decreases with length: even if parameters are learnt that output precise results on small sequences as
\begin{align*}
f_\theta((t_1))&=\mathbf B_{t_1,:}+\mathbf D_{t_1,:}\\\text{ and }
f_\theta((t_1,t_2))&=\mathbf B_{t_2,:}+\frac{\mathbf D_{t_1,:}+\mathbf D_{t_2,:}}{2},
\end{align*}
in longer sequences, softmax attention dilutes the next token logit displacements $\mathbf D$.

Note that the dilution comes from the multiplier $\frac{1}{s}$ by sequence length. Therefore, we can resharpen attention by applying an RMSNorm \cite{zhang2019root} operation to the output. We will verify that this operation indeed fosters length generalization in the experimental, second half of our paper, both in case of the Set Complement Task, and further formal language tasks.

\subsection{Training: NLL of One-Hot Sampled Target Distribution}
\label{subsection:nll}

\begin{figure*}[ht]
  \vskip 0.2in
  \begin{center}
    \centerline{\includegraphics{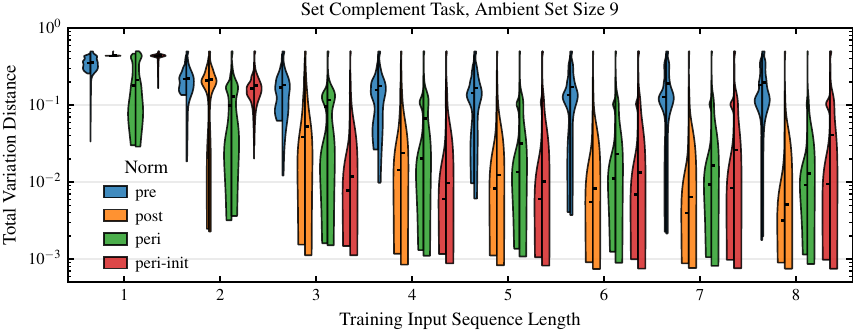}}
    \caption{
      Distribution of the validation TVD at the best checkpoint of each of the 1000 $d=8$, $d_k=1$, $d_v=8$ minimal transformers trained on the Set Complement Task with ambient set size $v=9$ and varying input sequence length $s$. In each violin plot, the left half shows the values without BEMA, and the right half with BEMA.
    }
    \label{figure:set complement 9}
  \end{center}
\end{figure*}

We seek to get models with next token probability distribution $p_\theta(\mathbf t)=\mathop{\mathrm{softmax}}(f_\theta(\mathbf t))\in\Delta^{v-1}$ approximating the uniform distribution $p^*(\mathbf t)$ on tokens absent from the input sequence $\mathbf t=(t_1,\dotsc,t_s)\in\mathbb X$, see Equation (\ref{equation:perfect solution}). However, in our study of training dynamics, we intend to follow the general practice in training generative language models: we sample an extra token $u\in\mathbb V\backslash[\mathbf t]$ and the model receives as loss the negative log likelihood 
$$
\mathop{\mathrm{NLL}}(\mathbf t, u;\theta)=-\log p_\theta(\mathbf t)_u
$$
between the predicted next token logit distributions after the input sequence $\mathbf t$ and the one-hot categorical distribution at token $u$.

As we have $u\notin[\mathbf t]$, the concatenation $(t_1,\dotsc,t_s,u)\in\mathbb V^{s+1}$ has no repetitions and is thus a valid input sequence if and only if we have $s<v-1$. We let $\bar{\mathbb X}$ denote the collection of sequences of tokens from $\mathbb V$ without repetition. That is, for a sequence $\mathbf t\in\bar{\mathbb X}$, we have $\mathbf t\in\mathbb X$ if and only if $|\mathbf t|<v$. For a prefix length $1\le s'\le|\mathbf t|$, the \emph{prefix sequence of length $s'$} is $\mathbf t_{:s'}=(t_1,\dotsc,t_{s'})\in\bar{\mathbb X}$.

In our case of main interest, that of length generalization, the lengths $s$ of input sequences are much smaller than the ambient set size $v$. This means that the number $v-s$ of possible target next tokens $u\in\mathbb V\backslash[\mathbf t]$ is large. Therefore, the model will receive training signals with high noise, thus slowing training.

\subsection{Mitigating Slowdown from Noisy Sampled Targets with BEMA}
\label{subsection:bema}

\begin{figure*}[ht]
  \vskip 0.2in
  \begin{center}
    \centerline{\includegraphics{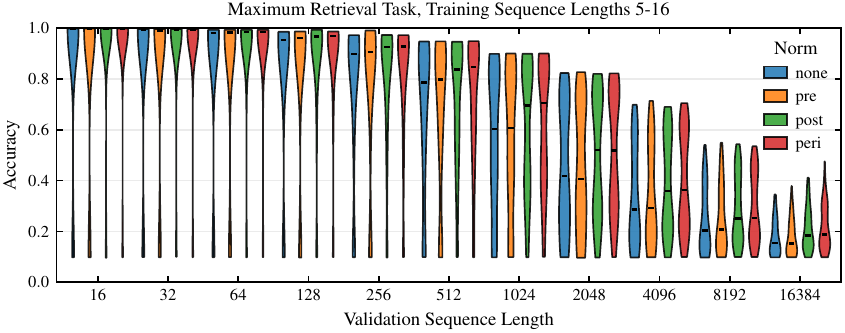}}
    \caption{
      Distribution of the validation accuracies at the best checkpoints of each of the 100 single attention head architectures trained on the Maximum Retrieval Task with training input sequence lengths sampled uniformly from the set $\{5,\dotsc,16\}$.
    }
    \label{figure:max retrieval}
  \end{center}
\end{figure*}

We hypothesize that Exponential Moving Average (EMA), a general remedy for gradient noise-induced slowdown, may provide a mitigation in this setting. We use Bias-Corrected Exponential Moving Average (BEMA) \cite{block2025ema}, that we now introduce for completeness:

We use three hyperparameters:  the \emph{EMA lag} $\rho$, the \emph{EMA power} $\kappa$, and the \emph{BEMA power} $\eta$.
At training step $n\ge0$, let $\theta_n$ denote the parameter values. In particular, we denote the intial parameter values as $\theta_0$. In the context of EMA, we also call them \emph{training parameter values}. Then the \emph{EMA parameter values} $\theta^\mathrm{EMA}_n$ are inductively defined as follows:
\begin{align*}
\theta^\mathrm{EMA}_0&=\theta_0\text{ and }\\
\theta^\mathrm{EMA}_{n+1}&=(1-\beta_n)\theta^\mathrm{EMA}_n + \beta_n\theta_{n+1}
\text{ where }\beta_n=(\rho+n)^{-\kappa}.
\end{align*}
We call $\beta_n$ an \emph{EMA weight}. Finally, at inference, we use the \emph{BEMA parameter values} $\theta^\mathrm{BEMA}_n$, that are defined as follows:
$$
\theta^\mathrm{BEMA}_n=\alpha_n(\theta_n-\theta_0) + \theta^\mathrm{EMA}_n\text{ where }
\alpha_n=(\rho+n)^{-\eta}.
$$
We call $\alpha_n$ a \emph{BEMA weight}.

Note that we only need to store, in addition to the most recent training parameter values $\theta_n$, the initial parameter values $\theta_0$, and the most recent EMA parameter values $\theta^\mathrm{EMA}_n$.

\section{Experimental Setup}\label{section:experimental setup}

In this Section, we detail the experimental setup in which we verify if the tools we propose indeed foster length generalization.

\subsection{Normalization Options}

In our survey of related work in Section \ref{section:related work}, we in part discuss how normalization layers can be used in transformer models. Following \cite{kim2025periln}, in our experiments, we consider the following four normalization options:
\begin{description}
  \item[pre-norm] Normalize the input of the residual blocks.
  \item[post-norm] Normalize residual stream after each residual update.
  \item[peri-norm] Normalize the input and the output of each residual block, and the input of the unembedding.
  \item[peri-init-norm] Besides the positions listed in peri-norm, also normalize the output of the embedding.
\end{description}

\subsection{Hyperparameter Distributions}

As we consider various architectural options and various datasets, in order to have a robust comparison, we need to run experiments with various hyperparameter configurations. See for example \cite{zhang2022adam} on the importance of tuning the first and second moment decay rates of Adam \cite{kingma2017adam}. See Appendix \ref{section:hyperparameter distribution} for the hyperparameter distribution.

\subsection{Main Metric: Total Variation Distance}\label{sunsection:metrics}

Let $\mathbf t\in\mathbb X$ be an input sequence. Recall that $f_\theta(\mathbf t)\in\mathbb R^v$ is the vector of predicted unnormalized next token logits to follow $\mathbf t$, and we let $p_\theta(\mathbf t)=\mathop{\mathrm{softmax}}(f_\theta(\mathbf t))$ denote the corresponding next token probabilities.

The most important metric which we utilize to measure how closely the predicted distribution $p_\theta(\mathbf t)$ approximates the target distribution $p^*(\mathbf t)$, in particular the uniform distribution on legal tokens in case of the Set Complement task (see Equation (\ref{equation:perfect solution})), is \emph{total variation distance (TVD)}:
\begin{equation}\label{equation:tvd}
\mathop{\mathrm{TVD}}(\mathbf t;\theta)=\frac12\sum_{t=1}^v|p_\theta(\mathbf t)_i-p^*(\mathbf t)_i|.
\end{equation}
In further tasks, we use additional metrics, which are discussed as needed.

\subsection{Dataloaders, Loss and Metric Aggregation}\label{subsection:dataloaders, loss and metric aggregation}

Both our training and validation dataloaders output minibatches of sequences $\mathbf T=(\mathbf t_1,\dotsc,\mathbf t_N)\in\bar{\mathbb X}^N$. The difference is in the length of the sequences: in order to study length generalization, the training dataloaders output shorter minibatches than the validation dataloaders.

Let $\mathbf T$ denote a training or validation minibatch. We follow the standard convention to aggregate the loss and the metrics not only by averaging across the minibatch entries,
but also across the sequential dimension.
Sometimes, we only aggregate over bins of sequence lengths, or report metrics on given sequence lengths.

\subsection{Training}

\begin{figure*}[ht]
  \vskip 0.2in
  \begin{center}
    \centerline{\includegraphics{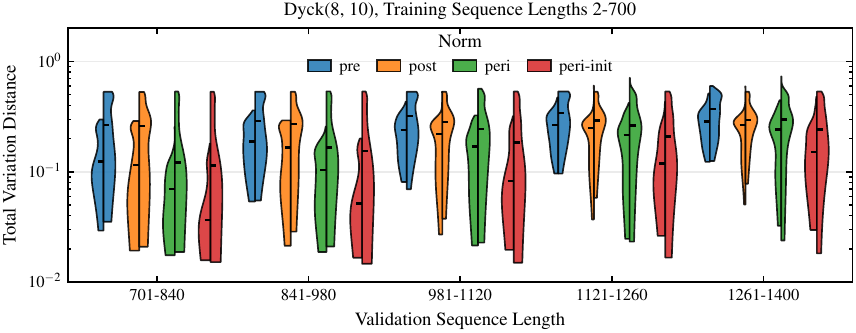}}
    \caption{
      Distribution of the TVD at the best checkpoints of each of the 100 2-layer, 64-dimensional, 4-headed transformers on the $\mathrm{Dyck}(8, 10)$ with maximum training input sequence length 700.
    }
    \label{figure:Dyck(8,10)}
  \end{center}
\end{figure*}

Following standard conventions, we initialize parameter matrices with normal distribution of std $\sigma=0.02$ and truncated at $2\sigma$. We use the AdamW \cite{loshchilov2018decoupled} optimizer. We follow the standard practice of disabling weight decay on embedding, and norm vectors; the latter decision is ablated in \cite{d'angelo2024why}. For learning rate schedule, we use linear warmup, and linear decay \cite{bergsma2025straight}. For each of our models, we report two sets of metrics: one for the training parameters $\theta$, and one for the BEMA parameters $\theta^\mathrm{BEMA}$, see Subsection \ref{subsection:bema}.

\section{Experiment Results}\label{section:experiment results}

\subsection{Set Complement Task}\label{subsection:sct results}

\begin{figure*}[ht]
  \vskip 0.2in
  \begin{center}
    \centerline{\includegraphics{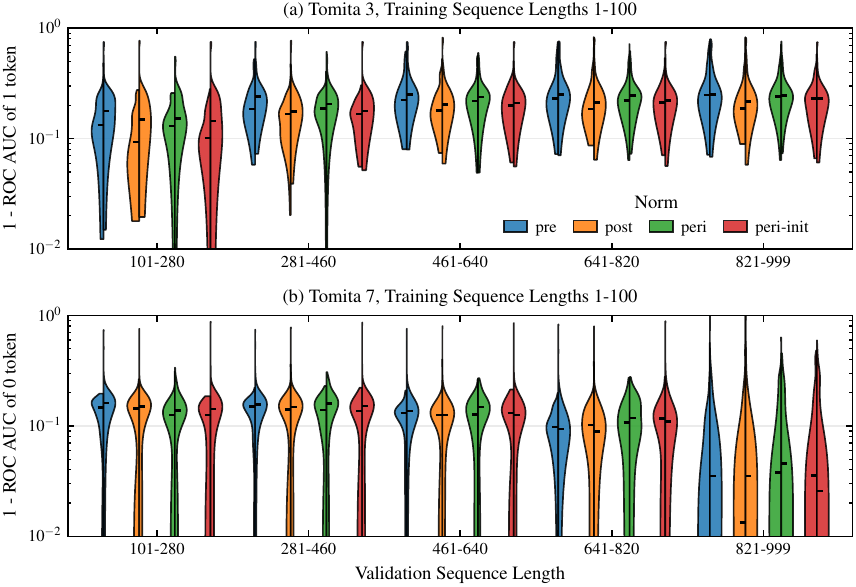}}
    \caption{
      Distribution of the ROC AUC scores at tokens 1 resp.~0 at the best checkpoints of each of the 100 2-layer, 64-dimensional, 4-headed transformers on the Tomita 3 resp.~7 languages with maximum training input sequence length 100.
    }
    \label{figure:Tomita 3, 7}
  \end{center}
\end{figure*}

For ambient set sizes $v=9,16$ and input sequence lengths $s=1,\dotsc,8$, we train minimal transformers (see Subsection \ref{subsection:minimal transformers}) with the minimal possible dimension $d=v-11$, $d_k=1$, $d_v=v-1$ permitted by Theorem \ref{theorem:bounds and length generalization}, for 1\,000\,000 training steps, on minibatches of size $N=256$.

See Figure \ref{figure:set complement 9} for the distribution of the best TVD of all our models in case $v=9$. One can see that all normalization methods that include some sort of post-normalization can length generalize from training input sequences of length $s\ge3$, peri-normalization can robustly length generalize from $s\ge2$ in concert with our Theorem \ref{theorem:bounds and length generalization}, and moreover, it can even length generalize to some extent with training input sequence length 1, thus overcoming our Theorem, that did not suppose post-normalization. We can also see that in most cases, BEMA can boost the best results a little further. See Appendix \ref{subsection:SCT appendix} for similar results in case $v=16,25$.

\subsection{Maximum Retrieval}\label{subsection:max retrieval results}

This task was designed in \cite{velickovic2025softmax} to showcase the attention dispersion effect in a simple setting. We sample $s$ set entries with a priority $\rho_i\sim\mathcal U([0,1))$ and a class $\kappa_i\sim\mathcal U(\{1,\dotsc,10\})$. We get input features $\mathbf x_i=(\rho_i|e_{\kappa_i})\in\mathbb R^{11}$. These are first embedded via a 2-layer MLP $\psi$ into a hidden representation $\mathbf h_i\in\mathbb R^{128}$. Then attention logits are formed as scalar products\footnote{Using a parameter vector $\mathbf q$ instead of another 2-layer MLP image $\psi_q(q)$ of a random value is our change.} $e_i=\mathbf q\cdot\left(\mathbf K\mathbf h_i\right)^T$. These in turn give aggregated values $\mathbf z=\sum_{i=1}^s\mathrm{softmax}(e_j)_i\mathbf V\mathbf h_i$, which are then transformed to a logit $\mathbf y=\phi(\mathbf z)\in\mathbb R^{10}$ by another 2-layer MLP. The model is trained to predict the class $i$ of the highest priority $\rho_i$.

We train 100 models on 100\,000 minibatches of 128 input sequences of lengths\footnote{In the original, the lengths in each minibatch are constant; we removed this restriction.} between 5 and 16 and then validated on longer sequences. We test the effect of including a normalization layer after the embedding and after value aggregation, respectively called pre- and post-normalization. We call using normalization at both positions peri-normalization. On Figure \ref{figure:max retrieval} we can see that we more robustly get higher accuracies with post- and peri-normalization, and the effect increases with sequence length.

\subsection{Bounded Dyck Languages}\label{subsection:dyck results}

In the Dyck language $\mathrm{Dyck}(k)$ with $k$ types, we have $k$ pairs of opening $(_i$ and closing $)_i$ parentheses. Then the valid sequences are the correctly nested sequences. In the theory of formal languages, these languages have central importance by the Chomsky--Sch\"utzenberger Representation Theorem \cite{chomsky1959algebraic}.

In the bounded Dyck language $\mathrm{Dyck}(k,m)$ of maximum depth $m$, there can be at most $m$ unclosed brackets at any time. It turns out that if we wanted to model hierarchical structures of clauses in ``Standard Average European'' languages, we could let $m=3$ \cite{karlsson2007constrains}.

We include Beginning of Sequence (BOS) and End of Sequence (EOS) tokens in the vocabulary. The first represents an empty sequence, and the latter a complete one. Note that if we train a generative model on random valid sequences, the predicted EOS logit can be used as a complete sequence classification signal.

Note that if we gave each valid token equal probability, then we would usually be in the depths $m-1$ and $m$, making EOS tokens very rare. Therefore, we follow prior work \cite{hewitt2020rnns} on benchmarking generative models on learning $\mathrm{Dyck}(k,m)$ in giving the depth increase and depth decrease / EOS events equal probability, thus making opening brackets have probability $\frac{1}{2k}$ when they are legal.

We train 100 2-layer transformers with 4 attention heads, 64 embedding dimensions, and Rotary Positional Encoding (RoPE) \cite{su2024roformer} on the task, for 100\,000 minibatches of size 64. In Figure \ref{figure:Dyck(8,10)}, we can see that, here too, normalization schemes that include post normalization have better length generalization results, while BEMA can boost the best metrics a little further. See Appendix \ref{subsection:extra dyck} for further results.

\subsection{Tomita Languages}\label{subsection:tomita results}

The Tomita languages are 7 formal languages on the vocabulary $\mathbb V=\{\mathrm{[BOS]}, 0, 1, \mathrm{[EOS]}\}$ created to benchmark how well a specific hill-climbing algorithm can find finite state automata that can recognize the languages \cite{tomita1982dynamic}. As the problem set features various long-dependency relationships, it has become an often used benchmark in training language models to recognize formal languages \cite{schmidhuber2001evaluating}.

Here, we consider languages 3 and 7. In 3, ``odd number of consecutive 1s are always followed by an even number of 0s'', while 7 ``has the regular expression $0^*1^*0^*1^*$''\cite{bhattamishra2020ability}. Note that if tokens 0, 1, and EOS had the same probability when legal, the sequences would terminate very early. So we set up the probabilities of EOS, and switches between 0 and 1, respectively, so that empirically about half of the sequences terminate by length 100.

Since, potentially for the very skewed token distributions, the TVD results do not give much signal, we show in Figure \ref{figure:Tomita 3, 7} the ROC AUC values of the tokens 1 and 0, respectively. That is, at each sequence length separately, we collect all validation entry pairs where the token is valid in one and not in the other, and measure the proportion of times that in the valid case the predicted token logit is higher. We can see that the peri-normalization methods are better on validation sequence lengths up to 460 on Tomita 3, while in other cases, they all produce similar results.

\section{Conclusion}

Length generalization, the ability of performing a task on contexts longer than those seen during training is an essential trait of autonomous agents. Transformers, the leading architecture of our time, usually use softmax attention, which is proven to exhibit attention dispersion, thus performance degrades with sequence length. Via a mechanistic study on the Set Complement Task, we discover that post-normalization, that is normalizing the output of attention blocks may mitigate this effect. We provide experimental evidence on various algorithmic tasks.

\section*{Acknowledgements}

Part of this research was conducted under the auspices of the Budapest Semesters in Mathematics program's
"Research Opportunities" initiative. P. Zs. was supported by the Ministry of Innovation and
Technology NRDI Office within the framework of the Artificial Intelligence National Laboratory (RRF-
2.3.1-21-2022-00004).

\section*{Impact Statement}

This paper presents work whose goal is to advance the field of Machine
Learning. There are many potential societal consequences of our work, none
which we feel must be specifically highlighted here.

\bibliography{post_norm_can_resharpen_attention}

\begin{thebibliography}{45}
\providecommand{\natexlab}[1]{#1}
\providecommand{\url}[1]{\texttt{#1}}
\expandafter\ifx\csname urlstyle\endcsname\relax
  \providecommand{\doi}[1]{doi: #1}\else
  \providecommand{\doi}{doi: \begingroup \urlstyle{rm}\Url}\fi

\bibitem[Bergsma et~al.(2025)Bergsma, Dey, Gosal, Gray, Soboleva, and Hestness]{bergsma2025straight}
Bergsma, S., Dey, N.~S., Gosal, G., Gray, G., Soboleva, D., and Hestness, J.
\newblock Straight to zero: Why linearly decaying the learning rate to zero works best for {LLM}s.
\newblock In \emph{The Thirteenth International Conference on Learning Representations}, 2025.
\newblock URL \url{https://openreview.net/forum?id=hrOlBgHsMI}.

\bibitem[Bhattamishra et~al.(2020)Bhattamishra, Ahuja, and Goyal]{bhattamishra2020ability}
Bhattamishra, S., Ahuja, K., and Goyal, N.
\newblock On the {A}bility and {L}imitations of {T}ransformers to {R}ecognize {F}ormal {L}anguages.
\newblock In Webber, B., Cohn, T., He, Y., and Liu, Y. (eds.), \emph{Proceedings of the 2020 Conference on Empirical Methods in Natural Language Processing (EMNLP)}, pp.\  7096--7116, Online, November 2020. Association for Computational Linguistics.
\newblock \doi{10.18653/v1/2020.emnlp-main.576}.
\newblock URL \url{https://aclanthology.org/2020.emnlp-main.576/}.

\bibitem[Block \& Zhang(2025)Block and Zhang]{block2025ema}
Block, A. and Zhang, C.
\newblock Ema without the lag: Bias-corrected iterate averaging schemes, 2025.
\newblock URL \url{https://arxiv.org/abs/2508.00180}.

\bibitem[Chiang \& Cholak(2022)Chiang and Cholak]{chiang2022overcoming}
Chiang, D. and Cholak, P.
\newblock Overcoming a theoretical limitation of self-attention.
\newblock In Muresan, S., Nakov, P., and Villavicencio, A. (eds.), \emph{Proceedings of the 60th Annual Meeting of the Association for Computational Linguistics (Volume 1: Long Papers)}, pp.\  7654--7664, Dublin, Ireland, May 2022. Association for Computational Linguistics.
\newblock \doi{10.18653/v1/2022.acl-long.527}.
\newblock URL \url{https://aclanthology.org/2022.acl-long.527/}.

\bibitem[Chomsky \& Schützenberger(1959)Chomsky and Schützenberger]{chomsky1959algebraic}
Chomsky, N. and Schützenberger, M.
\newblock The algebraic theory of context-free languages.
\newblock In Braffort, P. and Hirschberg, D. (eds.), \emph{Computer Programming and Formal Systems}, volume~26 of \emph{Studies in Logic and the Foundations of Mathematics}, pp.\  118--161. Elsevier, 1959.
\newblock \doi{https://doi.org/10.1016/S0049-237X(09)70104-1}.
\newblock URL \url{https://www.sciencedirect.com/science/article/pii/S0049237X09701041}.

\bibitem[D'Angelo et~al.(2024)D'Angelo, Andriushchenko, Varre, and Flammarion]{d'angelo2024why}
D'Angelo, F., Andriushchenko, M., Varre, A., and Flammarion, N.
\newblock Why do we need weight decay in modern deep learning?
\newblock In \emph{The Thirty-eighth Annual Conference on Neural Information Processing Systems}, 2024.
\newblock URL \url{https://openreview.net/forum?id=YrAxxscKM2}.

\bibitem[Ding et~al.(2021)Ding, Yang, Hong, Zheng, Zhou, Yin, Lin, Zou, Shao, Yang, and Tang]{ding2021cogview}
Ding, M., Yang, Z., Hong, W., Zheng, W., Zhou, C., Yin, D., Lin, J., Zou, X., Shao, Z., Yang, H., and Tang, J.
\newblock Cogview: Mastering text-to-image generation via transformers.
\newblock In Beygelzimer, A., Dauphin, Y., Liang, P., and Vaughan, J.~W. (eds.), \emph{Advances in Neural Information Processing Systems}, 2021.
\newblock URL \url{https://openreview.net/forum?id=cnWSyJNmeCE}.

\bibitem[Elhage et~al.(2021)Elhage, Nanda, Olsson, Henighan, Joseph, Mann, Askell, Bai, Chen, Conerly, DasSarma, Drain, Ganguli, Hatfield-Dodds, Hernandez, Jones, Kernion, Lovitt, Ndousse, Amodei, Brown, Clark, Kaplan, McCandlish, and Olah]{elhage2021mathematical}
Elhage, N., Nanda, N., Olsson, C., Henighan, T., Joseph, N., Mann, B., Askell, A., Bai, Y., Chen, A., Conerly, T., DasSarma, N., Drain, D., Ganguli, D., Hatfield-Dodds, Z., Hernandez, D., Jones, A., Kernion, J., Lovitt, L., Ndousse, K., Amodei, D., Brown, T., Clark, J., Kaplan, J., McCandlish, S., and Olah, C.
\newblock A mathematical framework for transformer circuits.
\newblock \emph{Transformer Circuits Thread}, 2021.
\newblock https://transformer-circuits.pub/2021/framework/index.html.

\bibitem[Gokaslan \& Cohen(2019)Gokaslan and Cohen]{gokaslan2019openweb}
Gokaslan, A. and Cohen, V.
\newblock Openwebtext corpus.
\newblock \url{http://Skylion007.github.io/OpenWebTextCorpus}, 2019.
\newblock Last accessed on October 8, 2025.

\bibitem[Hahn(2020)]{hahn2020theoretical}
Hahn, M.
\newblock Theoretical limitations of self-attention in neural sequence models.
\newblock \emph{Transactions of the Association for Computational Linguistics}, 8:\penalty0 156--171, 2020.
\newblock \doi{10.1162/tacl_a_00306}.
\newblock URL \url{https://aclanthology.org/2020.tacl-1.11/}.

\bibitem[Hanna et~al.(2023)Hanna, Liu, and Variengien]{hanna2023how}
Hanna, M., Liu, O., and Variengien, A.
\newblock How does gpt-2 compute greater-than?: Interpreting mathematical abilities in a pre-trained language model.
\newblock In Oh, A., Naumann, T., Globerson, A., Saenko, K., Hardt, M., and Levine, S. (eds.), \emph{Advances in Neural Information Processing Systems}, volume~36, pp.\  76033--76060. Curran Associates, Inc., 2023.
\newblock URL \url{https://proceedings.neurips.cc/paper_files/paper/2023/file/efbba7719cc5172d175240f24be11280-Paper-Conference.pdf}.

\bibitem[Hewitt et~al.(2020)Hewitt, Hahn, Ganguli, Liang, and Manning]{hewitt2020rnns}
Hewitt, J., Hahn, M., Ganguli, S., Liang, P., and Manning, C.~D.
\newblock {RNN}s can generate bounded hierarchical languages with optimal memory.
\newblock In Webber, B., Cohn, T., He, Y., and Liu, Y. (eds.), \emph{Proceedings of the 2020 Conference on Empirical Methods in Natural Language Processing (EMNLP)}, pp.\  1978--2010, Online, November 2020. Association for Computational Linguistics.
\newblock \doi{10.18653/v1/2020.emnlp-main.156}.
\newblock URL \url{https://aclanthology.org/2020.emnlp-main.156/}.

\bibitem[Huang et~al.(2025)Huang, Yang, Bhattamishra, Sarrof, Krebs, Zhou, Nakkiran, and Hahn]{huang2025a}
Huang, X., Yang, A., Bhattamishra, S., Sarrof, Y., Krebs, A., Zhou, H., Nakkiran, P., and Hahn, M.
\newblock A formal framework for understanding length generalization in transformers.
\newblock In \emph{The Thirteenth International Conference on Learning Representations}, 2025.
\newblock URL \url{https://openreview.net/forum?id=U49N5V51rU}.

\bibitem[Kalai \& Vempala(2024)Kalai and Vempala]{kalai2024calibrated}
Kalai, A.~T. and Vempala, S.~S.
\newblock Calibrated language models must hallucinate.
\newblock In \emph{Proceedings of the 56th Annual ACM Symposium on Theory of Computing}, STOC 2024, pp.\  160–171, New York, NY, USA, 2024. Association for Computing Machinery.
\newblock ISBN 9798400703836.
\newblock \doi{10.1145/3618260.3649777}.
\newblock URL \url{https://doi.org/10.1145/3618260.3649777}.

\bibitem[Karlsson(2007)]{karlsson2007constrains}
Karlsson, F.
\newblock Constraints on multiple center-embedding of clauses.
\newblock \emph{Journal of Linguistics}, 43\penalty0 (2):\penalty0 365–392, 2007.
\newblock \doi{10.1017/S0022226707004616}.

\bibitem[Kim et~al.(2025)Kim, Lee, Park, Oh, Kim, Yoo, Shin, Han, Shin, and Yoo]{kim2025periln}
Kim, J., Lee, B., Park, C., Oh, Y., Kim, B., Yoo, T., Shin, S., Han, D., Shin, J., and Yoo, K.~M.
\newblock Peri-{LN}: Revisiting normalization layer in the transformer architecture.
\newblock In \emph{Forty-second International Conference on Machine Learning}, 2025.
\newblock URL \url{https://openreview.net/forum?id=ci1S6wmXfO}.

\bibitem[Kingma \& Ba(2017)Kingma and Ba]{kingma2017adam}
Kingma, D.~P. and Ba, J.
\newblock Adam: A method for stochastic optimization, 2017.
\newblock URL \url{https://arxiv.org/abs/1412.6980}.

\bibitem[Kobayashi et~al.(2024)Kobayashi, Akram, and von Oswald]{kobayashi2024weight}
Kobayashi, S., Akram, Y., and von Oswald, J.
\newblock Weight decay induces low-rank attention layers.
\newblock In Globerson, A., Mackey, L., Belgrave, D., Fan, A., Paquet, U., Tomczak, J., and Zhang, C. (eds.), \emph{Advances in Neural Information Processing Systems}, volume~37, pp.\  4481--4510. Curran Associates, Inc., 2024.
\newblock URL \url{https://proceedings.neurips.cc/paper_files/paper/2024/file/084a67fb91826028f555e288f3adc9a4-Paper-Conference.pdf}.

\bibitem[Li et~al.(2023)Li, Hopkins, Bau, Vi{\'e}gas, Pfister, and Wattenberg]{li2023emergent}
Li, K., Hopkins, A.~K., Bau, D., Vi{\'e}gas, F., Pfister, H., and Wattenberg, M.
\newblock Emergent world representations: Exploring a sequence model trained on a synthetic task.
\newblock In \emph{The Eleventh International Conference on Learning Representations}, 2023.
\newblock URL \url{https://openreview.net/forum?id=DeG07_TcZvT}.

\bibitem[Li \& Liang(2021)Li and Liang]{li2021prefix}
Li, X.~L. and Liang, P.
\newblock Prefix-tuning: Optimizing continuous prompts for generation.
\newblock In Zong, C., Xia, F., Li, W., and Navigli, R. (eds.), \emph{Proceedings of the 59th Annual Meeting of the Association for Computational Linguistics and the 11th International Joint Conference on Natural Language Processing (Volume 1: Long Papers)}, pp.\  4582--4597, Online, August 2021. Association for Computational Linguistics.
\newblock \doi{10.18653/v1/2021.acl-long.353}.
\newblock URL \url{https://aclanthology.org/2021.acl-long.353/}.

\bibitem[Liu et~al.(2022)Liu, Hu, Lin, Yao, Xie, Wei, Ning, Cao, Zhang, Dong, Wei, and Guo]{liu2022swin}
Liu, Z., Hu, H., Lin, Y., Yao, Z., Xie, Z., Wei, Y., Ning, J., Cao, Y., Zhang, Z., Dong, L., Wei, F., and Guo, B.
\newblock Swin transformer v2: Scaling up capacity and resolution.
\newblock In \emph{2022 IEEE/CVF Conference on Computer Vision and Pattern Recognition (CVPR)}, pp.\  11999--12009, 2022.
\newblock \doi{10.1109/CVPR52688.2022.01170}.

\bibitem[Loshchilov \& Hutter(2019)Loshchilov and Hutter]{loshchilov2018decoupled}
Loshchilov, I. and Hutter, F.
\newblock Decoupled weight decay regularization.
\newblock In \emph{International Conference on Learning Representations}, 2019.
\newblock URL \url{https://openreview.net/forum?id=Bkg6RiCqY7}.

\bibitem[Lovering et~al.(2025)Lovering, Krumdick, Lai, Reddy, Ebner, Kumar, Koncel-Kedziorski, and Tanner]{lovering2025language}
Lovering, C., Krumdick, M., Lai, V.~D., Reddy, V., Ebner, S., Kumar, N., Koncel-Kedziorski, R., and Tanner, C.
\newblock Language model probabilities are $not$ calibrated in numeric contexts.
\newblock In Che, W., Nabende, J., Shutova, E., and Pilehvar, M.~T. (eds.), \emph{Proceedings of the 63rd Annual Meeting of the Association for Computational Linguistics (Volume 1: Long Papers)}, pp.\  29218--29257, Vienna, Austria, July 2025. Association for Computational Linguistics.
\newblock ISBN 979-8-89176-251-0.
\newblock \doi{10.18653/v1/2025.acl-long.1417}.
\newblock URL \url{https://aclanthology.org/2025.acl-long.1417/}.

\bibitem[Nanda et~al.(2023)Nanda, Lee, and Wattenberg]{nanda2023emergent}
Nanda, N., Lee, A., and Wattenberg, M.
\newblock Emergent linear representations in world models of self-supervised sequence models.
\newblock In Belinkov, Y., Hao, S., Jumelet, J., Kim, N., McCarthy, A., and Mohebbi, H. (eds.), \emph{Proceedings of the 6th BlackboxNLP Workshop: Analyzing and Interpreting Neural Networks for NLP}, pp.\  16--30, Singapore, December 2023. Association for Computational Linguistics.
\newblock \doi{10.18653/v1/2023.blackboxnlp-1.2}.
\newblock URL \url{https://aclanthology.org/2023.blackboxnlp-1.2/}.

\bibitem[Ouyang et~al.(2022)Ouyang, Wu, Jiang, Almeida, Wainwright, Mishkin, Zhang, Agarwal, Slama, Ray, Schulman, Hilton, Kelton, Miller, Simens, Askell, Welinder, Christiano, Leike, and Lowe]{ouyang2022training}
Ouyang, L., Wu, J., Jiang, X., Almeida, D., Wainwright, C., Mishkin, P., Zhang, C., Agarwal, S., Slama, K., Ray, A., Schulman, J., Hilton, J., Kelton, F., Miller, L., Simens, M., Askell, A., Welinder, P., Christiano, P.~F., Leike, J., and Lowe, R.
\newblock Training language models to follow instructions with human feedback.
\newblock In Koyejo, S., Mohamed, S., Agarwal, A., Belgrave, D., Cho, K., and Oh, A. (eds.), \emph{Advances in Neural Information Processing Systems}, volume~35, pp.\  27730--27744. Curran Associates, Inc., 2022.
\newblock URL \url{https://proceedings.neurips.cc/paper_files/paper/2022/file/b1efde53be364a73914f58805a001731-Paper-Conference.pdf}.

\bibitem[Peng et~al.(2024)Peng, Quesnelle, Fan, and Shippole]{peng2024yarn}
Peng, B., Quesnelle, J., Fan, H., and Shippole, E.
\newblock Ya{RN}: Efficient context window extension of large language models.
\newblock In \emph{The Twelfth International Conference on Learning Representations}, 2024.
\newblock URL \url{https://openreview.net/forum?id=wHBfxhZu1u}.

\bibitem[Rush \& Weiss(2023)Rush and Weiss]{rush2023thinking}
Rush, A.~M. and Weiss, G.
\newblock Thinking like transformers.
\newblock In \emph{The Second Blogpost Track at ICLR 2023}, 2023.
\newblock URL \url{https://openreview.net/forum?id=djS_CaOq2F}.

\bibitem[Schmidhuber et~al.(2001)Schmidhuber, Hochreiter, and Bengio]{schmidhuber2001evaluating}
Schmidhuber, J., Hochreiter, S., and Bengio, Y.
\newblock Evaluating benchmark problems by random guessing.
\newblock In Kremer, S.~C. and Kolen, J.~F. (eds.), \emph{A Field Guide to Dynamical Recurrent Neural Networks}, pp.\  231--236. Wiley-IEEE Press, 2001.

\bibitem[Shlegeris et~al.(2024)Shlegeris, Roger, Chan, and McLean]{shlegeris2024language}
Shlegeris, B., Roger, F., Chan, L., and McLean, E.
\newblock Language models are better than humans at next-token prediction.
\newblock \emph{Transactions on Machine Learning Research}, 2024.
\newblock ISSN 2835-8856.
\newblock URL \url{https://openreview.net/forum?id=RNsnSLdmV7}.

\bibitem[Su et~al.(2024)Su, Ahmed, Lu, Pan, Bo, and Liu]{su2024roformer}
Su, J., Ahmed, M., Lu, Y., Pan, S., Bo, W., and Liu, Y.
\newblock Roformer: Enhanced transformer with rotary position embedding.
\newblock \emph{Neurocomputing}, 568:\penalty0 127063, 2024.
\newblock ISSN 0925-2312.
\newblock \doi{https://doi.org/10.1016/j.neucom.2023.127063}.
\newblock URL \url{https://www.sciencedirect.com/science/article/pii/S0925231223011864}.

\bibitem[Sun et~al.(2024)Sun, Chen, Kolter, and Liu]{sun2024massive}
Sun, M., Chen, X., Kolter, J.~Z., and Liu, Z.
\newblock Massive activations in large language models.
\newblock \emph{arXiv preprint arXiv:2402.17762}, 2024.

\bibitem[Tomita(1982)]{tomita1982dynamic}
Tomita, M.
\newblock Dynamic construction of finite-state automata from examples using hill-climbing.
\newblock \emph{Proc. Fourth Annual Cognitive Science Conference, Ann Arbor, Mi, 1982}, pp.\  105--108, 1982.
\newblock URL \url{https://cir.nii.ac.jp/crid/1570854176062485760}.

\bibitem[Vaswani et~al.(2017)Vaswani, Shazeer, Parmar, Uszkoreit, Jones, Gomez, Kaiser, and Polosukhin]{vaswani2017attention}
Vaswani, A., Shazeer, N., Parmar, N., Uszkoreit, J., Jones, L., Gomez, A.~N., Kaiser, L.~u., and Polosukhin, I.
\newblock Attention is all you need.
\newblock In Guyon, I., Luxburg, U.~V., Bengio, S., Wallach, H., Fergus, R., Vishwanathan, S., and Garnett, R. (eds.), \emph{Advances in Neural Information Processing Systems}, volume~30. Curran Associates, Inc., 2017.
\newblock URL \url{https://proceedings.neurips.cc/paper_files/paper/2017/file/3f5ee243547dee91fbd053c1c4a845aa-Paper.pdf}.

\bibitem[Veli{\v{c}}kovi{\'c} et~al.(2025)Veli{\v{c}}kovi{\'c}, Perivolaropoulos, Barbero, and Pascanu]{velickovic2025softmax}
Veli{\v{c}}kovi{\'c}, P., Perivolaropoulos, C., Barbero, F., and Pascanu, R.
\newblock Softmax is not enough (for sharp size generalisation).
\newblock In \emph{Forty-second International Conference on Machine Learning}, 2025.
\newblock URL \url{https://openreview.net/forum?id=S4JmmpnSPy}.

\bibitem[Wallace et~al.(2019)Wallace, Feng, Kandpal, Gardner, and Singh]{wallace2019universal}
Wallace, E., Feng, S., Kandpal, N., Gardner, M., and Singh, S.
\newblock Universal adversarial triggers for attacking and analyzing {NLP}.
\newblock In Inui, K., Jiang, J., Ng, V., and Wan, X. (eds.), \emph{Proceedings of the 2019 Conference on Empirical Methods in Natural Language Processing and the 9th International Joint Conference on Natural Language Processing (EMNLP-IJCNLP)}, pp.\  2153--2162, Hong Kong, China, November 2019. Association for Computational Linguistics.
\newblock \doi{10.18653/v1/D19-1221}.
\newblock URL \url{https://aclanthology.org/D19-1221/}.

\bibitem[Wang et~al.(2025)Wang, Zhu, and Shi]{wang2025distribution}
Wang, H., Zhu, Z., and Shi, F.
\newblock Distribution prompting: Understanding the expressivity of language models through the next-token distributions they can produce, 2025.
\newblock URL \url{https://arxiv.org/abs/2505.12244}.

\bibitem[Wang et~al.(2023)Wang, Variengien, Conmy, Shlegeris, and Steinhardt]{wang2023interpretability}
Wang, K.~R., Variengien, A., Conmy, A., Shlegeris, B., and Steinhardt, J.
\newblock Interpretability in the wild: a circuit for indirect object identification in {GPT}-2 small.
\newblock In \emph{The Eleventh International Conference on Learning Representations}, 2023.
\newblock URL \url{https://openreview.net/forum?id=NpsVSN6o4ul}.

\bibitem[Wei et~al.(2022)Wei, Wang, Schuurmans, Bosma, Ichter, Xia, Chi, Le, and Zhou]{wei2022chain}
Wei, J., Wang, X., Schuurmans, D., Bosma, M., Ichter, B., Xia, F., Chi, E.~H., Le, Q.~V., and Zhou, D.
\newblock Chain-of-thought prompting elicits reasoning in large language models.
\newblock In \emph{Proceedings of the 36th International Conference on Neural Information Processing Systems}, NIPS '22, Red Hook, NY, USA, 2022. Curran Associates Inc.
\newblock ISBN 9781713871088.

\bibitem[Wu et~al.(2025)Wu, Wang, Xiao, Peng, and Fu]{wu2025retrieval}
Wu, W., Wang, Y., Xiao, G., Peng, H., and Fu, Y.
\newblock Retrieval head mechanistically explains long-context factuality.
\newblock In \emph{The Thirteenth International Conference on Learning Representations}, 2025.
\newblock URL \url{https://openreview.net/forum?id=EytBpUGB1Z}.

\bibitem[Xiong et~al.(2020)Xiong, Yang, He, Zheng, Zheng, Xing, Zhang, Lan, Wang, and Liu]{xiong2020layer}
Xiong, R., Yang, Y., He, D., Zheng, K., Zheng, S., Xing, C., Zhang, H., Lan, Y., Wang, L., and Liu, T.
\newblock On layer normalization in the transformer architecture.
\newblock In III, H.~D. and Singh, A. (eds.), \emph{Proceedings of the 37th International Conference on Machine Learning}, volume 119 of \emph{Proceedings of Machine Learning Research}, pp.\  10524--10533. PMLR, 13--18 Jul 2020.
\newblock URL \url{https://proceedings.mlr.press/v119/xiong20b.html}.

\bibitem[Yang \& Chiang(2024)Yang and Chiang]{yang2024counting}
Yang, A. and Chiang, D.
\newblock Counting like transformers: Compiling temporal counting logic into softmax transformers.
\newblock In \emph{First Conference on Language Modeling}, 2024.
\newblock URL \url{https://openreview.net/forum?id=FmhPg4UJ9K}.

\bibitem[Yao et~al.(2021)Yao, Peng, Papadimitriou, and Narasimhan]{yao2021self}
Yao, S., Peng, B., Papadimitriou, C., and Narasimhan, K.
\newblock Self-attention networks can process bounded hierarchical languages.
\newblock In Zong, C., Xia, F., Li, W., and Navigli, R. (eds.), \emph{Proceedings of the 59th Annual Meeting of the Association for Computational Linguistics and the 11th International Joint Conference on Natural Language Processing (Volume 1: Long Papers)}, pp.\  3770--3785, Online, August 2021. Association for Computational Linguistics.
\newblock \doi{10.18653/v1/2021.acl-long.292}.
\newblock URL \url{https://aclanthology.org/2021.acl-long.292/}.

\bibitem[Zhang \& Sennrich(2019)Zhang and Sennrich]{zhang2019root}
Zhang, B. and Sennrich, R.
\newblock Root mean square layer normalization.
\newblock In Wallach, H., Larochelle, H., Beygelzimer, A., d\textquotesingle Alch\'{e}-Buc, F., Fox, E., and Garnett, R. (eds.), \emph{Advances in Neural Information Processing Systems}, volume~32. Curran Associates, Inc., 2019.
\newblock URL \url{https://proceedings.neurips.cc/paper_files/paper/2019/file/1e8a19426224ca89e83cef47f1e7f53b-Paper.pdf}.

\bibitem[Zhang et~al.(2022)Zhang, Chen, Shi, Sun, and Luo]{zhang2022adam}
Zhang, Y., Chen, C., Shi, N., Sun, R., and Luo, Z.-Q.
\newblock Adam can converge without any modification on update rules.
\newblock In Oh, A.~H., Agarwal, A., Belgrave, D., and Cho, K. (eds.), \emph{Advances in Neural Information Processing Systems}, 2022.
\newblock URL \url{https://openreview.net/forum?id=l5UNyaHqFdO}.

\bibitem[Zhou et~al.(2024)Zhou, Bradley, Littwin, Razin, Saremi, Susskind, Bengio, and Nakkiran]{zhou2024what}
Zhou, H., Bradley, A., Littwin, E., Razin, N., Saremi, O., Susskind, J.~M., Bengio, S., and Nakkiran, P.
\newblock What algorithms can transformers learn? a study in length generalization.
\newblock In \emph{The Twelfth International Conference on Learning Representations}, 2024.
\newblock URL \url{https://openreview.net/forum?id=AssIuHnmHX}.

\end{thebibliography}
\bibliographystyle{icml2026}

\newpage
\appendix
\onecolumn
\section{Proof of Theorem \ref{theorem:bounds and length generalization}}\label{proof of Theorem}

The following Lemma is a key component of the tight dimension bounds:

\begin{lemma}\label{lemma:rank n-1} Let $\mathbf u,\mathbf v,\mathbf w\in\mathbb R^n$ be $n$-dimensional vectors. Suppose that we have $w_i<0$ for all indices $1\le i\le n$. Then the matrix $\mathbf A:=\boldsymbol 1\mathbf u^T+\mathbf v\boldsymbol1^T + \mathop{\mathrm{diag}}(\mathbf w)$ has rank at least $n-1$.

\end{lemma}

\begin{proof} It is enough to show that the matrix $\mathbf A$ is injective on the 1-codimensional subspace $Z:=\{\mathbf x\in\mathbb R^n:\sum_{i=1}^nx_i=0\}$. Take $\mathbf x\in Z$ and suppose that we have $\mathbf A\mathbf x=\boldsymbol0$. Let $\alpha=\mathbf u^T\mathbf x$. Then for each $1\le i\le n$, we have
$$
0=(\mathbf A\mathbf x)_i=\alpha+v_i\sum_{i=1}^nx_i+w_ix_i=\alpha+w_ix_i.
$$
If $\alpha=0$, then as $w_i<0$ for all $1\le i\le n$, we get $\mathbf x=\boldsymbol0$. Otherwise, we get
$$
0=\sum_{i=1}^nx_i=-\alpha\sum_{i=1}^n\frac{1}{w_i}<0,
$$
a contradiction.

\end{proof}

\begin{proof}[Proof of Theorem \ref{theorem:bounds and length generalization}] (a) In terms of the matrices $\mathbf B$ and $\mathbf D$, the fact that the model $f_\theta$ has precision $C$ at length 1 reads as, for distinct tokens $t,u,v\in\mathbb V$:
  \begin{align}
    b_{t,u}+d_{t,u}&>b_{t,t}+d_{t,t}+C\label{equation:logit displacement length 1}\\
    b_{t,u}+d_{t,u}&=b_{t,v}+d_{t,v}\label{equation:logit equality length 1}
  \end{align}
These conditions imply the conditions of Lemma \ref{lemma:rank n-1} for the matrix $\mathbf B+\mathbf D$, thus showing that we have $\mathop{\mathrm{rank}}(\mathbf B+\mathbf D)\ge v-1$.

(b) In terms of the matrices $\mathbf B$ and $\mathbf D$, the fact that the model $f_\theta$ has precision $C$ at length 2 reads as, for distinct tokens $t,u,v,w\in\mathbb V$:
\begin{align}
2b_{t,v}+d_{t,v}+d_{u,v}&>2b_{t,t}+d_{t,t}+d_{u,t}+2C\label{equation:logit displacement length 2 diagonal}\\
2b_{t,v}+d_{t,v}+d_{u,v}&>2b_{t,u}+d_{t,u}+d_{u,u}+2C\label{equation:logit displacement length 2 nondiagonal}\\
2b_{t,v}+d_{t,v}+d_{u,v}&=2b_{t,w}+d_{t,w}+d_{u,w}\label{equation:logit equality length 2}.
\end{align}

Equations (\ref{equation:logit equality length 1}) and (\ref{equation:logit equality length 2}) show that for all distinct tokens $t,u,v\in\mathbb V$: the difference $d_{t,v}-d_{u,v}$ is constant in $v$. Let us denote this by $\alpha_{t,u}$, and let $\alpha_{t,t}:=0$.

Let us fix $r\in\mathbb V$ and let $\mathbf a,\mathbf c\in\mathbb R^v$ be defined by $a_t=\alpha_{t,r},c_t=d_{r,t}$ for $t\in\mathbb V$. Then for all distinct $t,u\in\mathbb V$: we have $d_{t,u}=a_t+c_u$. Moreover, Constraints (\ref{equation:logit equality length 1}) and (\ref{equation:logit displacement length 2 nondiagonal}) show that we have $d_{t,t}-a_t-c_t<0$. Therefore, Lemma \ref{lemma:rank n-1} shows that we have $\mathop{\mathrm{rank}}(\mathbf D)\ge v-1$.

(c) Let us prove that $f_\theta$ has precision $\frac{2}{ s}C>0$ at length $ s$ by induction on $1\le  s<v$. By assumption, the induction hypothesis holds for $ s=1,2$. Let us assume that it holds for $ s$, that is, the following constraints are satisfied, for distinct tokens $t_1,\dotsc,t_ s,u,v\in\mathbb V$, and indices $1\le i< s$:
\begin{align}
 s b_{t_ s,u}+d_{t_1,u}+\dotsb+d_{t_ s,u}&> s b_{t_ s,t_ s}+d_{t_1,t_ s}+\dotsb+d_{t_ s,t_ s}+2C\label{equation:logit displacement length l diagonal}\\
 s b_{t_ s,u}+d_{t_1,u}+\dotsb+d_{t_ s,u}&> s b_{t_ s,t_i}+d_{t_1,t_i}+\dotsb+d_{t_ s,t_i}+2C\label{equation:logit displacement length l nondiagonal}\\
 s b_{t_ s,u}+d_{t_1,u}+\dotsb+d_{t_ s,u}&= s b_{t_ s,w}+d_{t_1,w}+\dotsb+d_{t_ s,w}.\label{equation:logit equality length l}
\end{align}
Let us undertake proving the induction step. By Inequalities (\ref{equation:logit displacement length 2 diagonal}) and (\ref{equation:logit displacement length l diagonal}), we get
\begin{align*}
&(2b_{t_{ s+1},u}+d_{t_1,u}+d_{t_{ s+1},u})
+( s b_{t_{ s+1},u}+d_{t_2,u}+\dotsb+d_{t_{ s+1},u})\\
>&(2b_{t_{ s+1},t_{ s+1}}+d_{t_1,t_{ s+1}}+d_{t_{ s+1},t_{ s+1}})\\
&+( s b_{t_{ s+1},t_{ s+1}}+d_{t_2,t_{ s+1}}+\dotsb+d_{t_{ s+1},t_{ s+1}})
+2C+2 C,
\end{align*}
which by Inequality (\ref{equation:precision upper bound}) yields
\begin{align*}
&( s+1)b_{t_{ s+1},u}+d_{t_1,u}+\dotsb+d_{t_{ s+1},u}\\>&
( s+1)b_{t_{ s+1},t_{ s+1}}+d_{t_1,t_{ s+1}}+\dotsb+d_{t_{ s+1},t_{ s+1}}+2C.
\end{align*}
Then note that by Equation (\ref{equation:logit equality length 1}), Equation (\ref{equation:logit equality length 2}) is equivalent to the following equation:
\begin{equation}\label{equation:logit equality length 2 variant 2}
b_{t_{ s+1},u}+d_{1,u}=d_{t_{ s+1},w}+d_{1,w}.
\end{equation}
With this and Inequality (\ref{equation:logit displacement length l nondiagonal}), we get
\begin{align*}
&(b_{t_{ s+1},u}+d_{1,u}+d_{t_{ s+1},u})
+( s b_{t_{ s+1},u}+d_{t_2,u}+\dotsb+d_{t_{ s+1},u})\\
>&(b_{t_{ s+1},t_j}+d_{1,u}+d_{t_{ s+1},t_j})
+( s b_{t_{ s+1},t_j}+d_{t_2,u}+\dotsb+d_{t_{ s+1},t_j})+ 2C
\end{align*}
Finally, Equalities (\ref{equation:logit equality length 2 variant 2}) and (\ref{equation:logit equality length l}) yield
\begin{align*}
&(b_{t_{ s+1},u}+d_{1,u}+d_{t_{ s+1},u})
+( s b_{t_{ s+1},u}+d_{t_2,u}+\dotsb+d_{t_{ s+1},u})\\
=&(b_{t_{ s+1},v}+d_{1,u}+d_{t_{ s+1},v})
+( s b_{t_{ s+1},v}+d_{t_2,u}+\dotsb+d_{t_{ s+1},v})
\end{align*}

\end{proof}

\section{Hyperparameter Distribution for Random Search}\label{section:hyperparameter distribution}

The following table describes the distribution from which we drew the hyperparameter configurations for our experiments.

\begin{table}[h]
	\caption{Hyperparameter distributions for the random search}
	\label{table:hyperparameter distributions}
  \begin{center}
	\begin{tabular}{rll}\toprule
		\textit{Hyperparameter} & \textit{Distribution} & \textit{Range} \\
    \midrule
		\multicolumn{3}{l}{\textit{Model}} \\
    \midrule
    RMSNorm $\epsilon$ & $10^{\mathcal U[-10,-4]}$ & $[10^{-10},10^{-4}]$ \\
    \midrule
		\multicolumn{3}{l}{\textit{AdamW}} \\
    \midrule
1.~moment decay $\beta_1$ & $1-10^{\mathcal U[-2,0]}$ & $[0, 1-10^{-2}]$ \\
2.~moment decay $\beta_2$ & $1-10^{\mathcal U[-1,-5]}$ & $[1-10^{-1}, 1-10^{-5}]$ \\
weight decay $\lambda$ & $10^{\mathcal U[-6,0]}$ & $[10^{-6}, 1]$ \\
AdamW $\epsilon$ & $10^{\mathcal U[-12,-8]}$ & $[10^{-12}, 10^{-8}]$ \\
max gradient norm & $10^{\mathcal U[-2, 2]}$ & $[10^{-2},10^2]$ \\
    \midrule
		\multicolumn{3}{l}{\textit{Learning Rate Schedule}} \\
    \midrule
peak learning rate $\eta$ & $10^{\mathcal U[-5,-2]}$ & $[10^{-5}, 10^{-2}]$ \\
warmup steps & $\lfloor 10^{\mathcal U[-2,4]}\rfloor$ & $[0,10^4]$ \\
multiplier at end & $10^{\mathcal U[-4,0]}$ & $[10^{-4},1]$ \\
    \midrule
		\multicolumn{3}{l}{\textit{BEMA}} \\
    \midrule
BEMA power $\eta$ & $\mathcal U[0,1]$ & $[0,1]$ \\
EMA lag $\rho$ & $10^{\mathcal U[0,10]}$ & $[1,10^{10}]$ \\
EMA power $\kappa$ & $\mathcal U[0,1]$ & $[0,1]$ \\
    \bottomrule
	\end{tabular}
\end{center}
\end{table}

\section{Further Experimental Results}\label{secion:further experiments}

\subsection{Set Complement Task}\label{subsection:SCT appendix}

With ambient set size $v=16,25$, we observe the same effect: including normalization of attention block output gives length generalization a big boost, and BEMA can further improve the best results.

\begin{figure*}[ht]
  \vskip 0.2in
  \begin{center}
    \centerline{\includegraphics{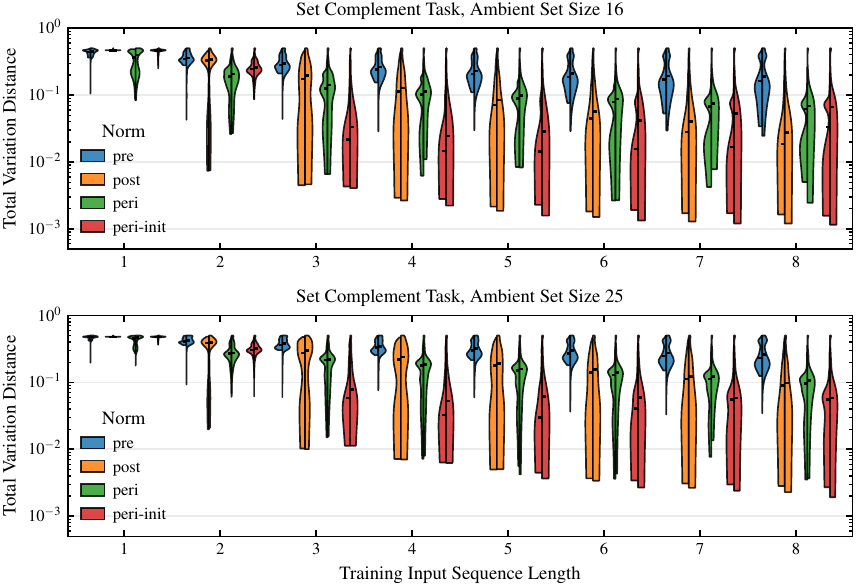}}
    \caption{
      Distribution of the validation TVD at the best checkpoint of each of the 1000 $d=v-1$, $d_k=1$, $d_v=v-1$ minimal transformers trained on the Set Complement Task with ambient set sizes $v=16,25$ and varying input sequence length $s$. In each violin plot, the left half shows the values without BEMA, and the right half with BEMA.
    }
    \label{figure:set complement 16}
  \end{center}
\end{figure*}

\subsection{$\mathrm{Dyck}(k,m)$}\label{subsection:extra dyck}

We run further experiments on bounded Dyck languages. In all cases, we can observe similarly that peri-normalization schemes offer the best length generalization and pre-normalization the worst. We can furthermore observe a more beneficial effect of BEMA for larger numbers $k$ of types, probably due to the increased sampling noise that comes with more opening brackets being possible. Moreover, we see worse TVD with smaller maximum depth $m$, probably indicating that the EOS logits are not well calibrated. To show that, in Figure \ref{figure:dyck extra roc auc}, we display the ROC AUC scores of the EOS tokens, showing that with peri-normalization, the models can robustly classify complete sequences.

\begin{figure*}[ht]
  \vskip 0.2in
  \begin{center}
    \centerline{\includegraphics{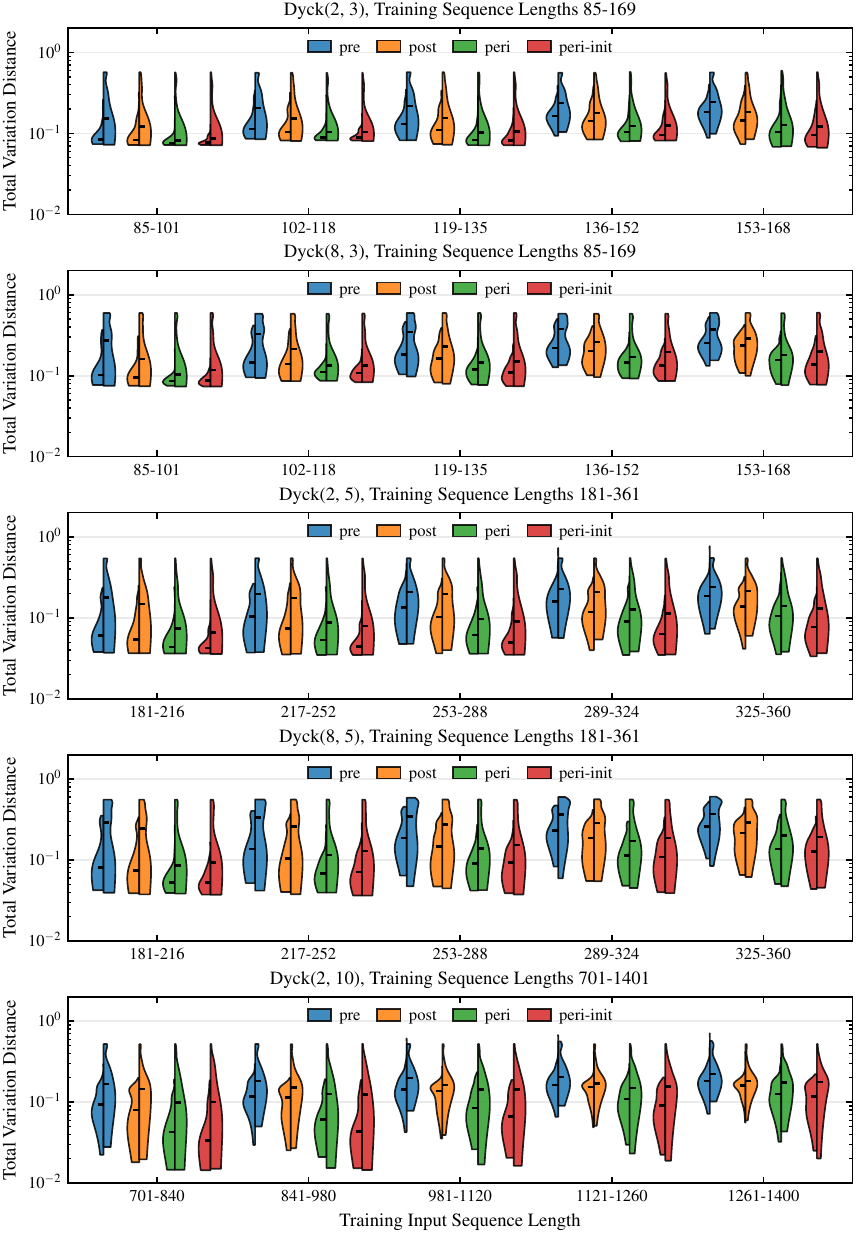}}
    \caption{
      Distribution of the TVD at the best checkpoints of each of the 100 2-layer, 64-dimensional, 4-headed transformers on further $\mathrm{Dyck}(k,m)$ tasks
    }
    \label{figure:dyck extra tvd}
  \end{center}
\end{figure*}

\begin{figure*}[ht]
  \vskip 0.2in
  \begin{center}
    \centerline{\includegraphics{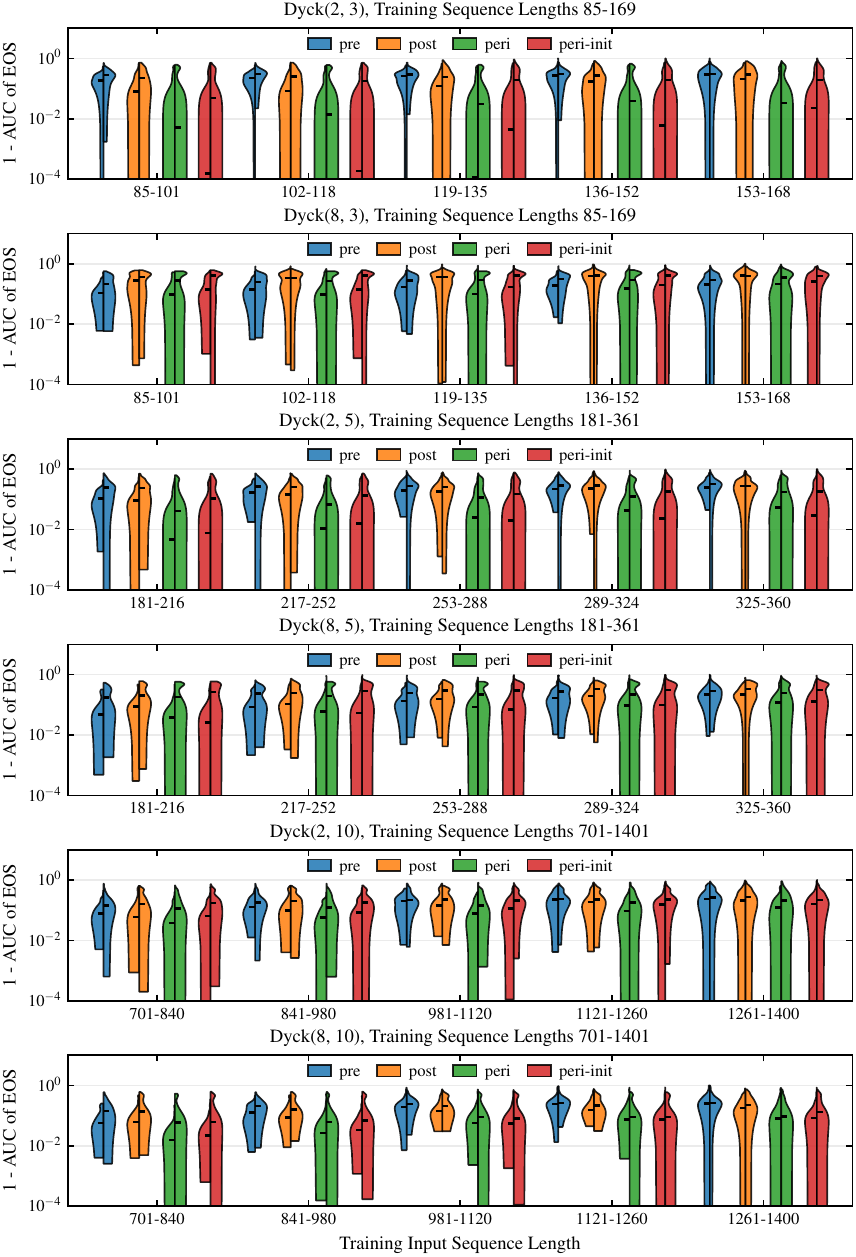}}
    \caption{
      Distribution of the ROC AUC of the EOS token at the best checkpoints of each of the 100 2-layer, 64-dimensional, 4-headed transformers on further $\mathrm{Dyck}(k,m)$ tasks
    }
    \label{figure:dyck extra roc auc}
  \end{center}
\end{figure*}

\end{document}